\documentclass[12pt]{article}
\usepackage{epsfig}
\usepackage{natbib}
\usepackage{times}
\usepackage{amsmath,amsthm,amsfonts,amssymb,mathrsfs}
\usepackage[fleqn,tbtags]{mathtools}

\newcommand{\bproof}{\begin{proof}}
\newcommand{\eproof}{\end{proof}}

\DeclareMathOperator\rank{rank}
\DeclareMathOperator\trace{tr}

\DeclareMathOperator\Span{span}

\begin{document}
\bibliographystyle{plainnat}
\makeatletter

\renewcommand{\theequation}{\thesection.\arabic{equation}}
\numberwithin{equation}{section}
\renewcommand{\bar}{\overline}
\newtheorem{theorem}{Theorem}[section]
\newtheorem{question}{Question}[section]
\newtheorem{proposition}{Proposition}[section]
\newtheorem{lemma}{Lemma}[section]
\newtheorem{corollary}{Corollary}[section]
\newtheorem{definition}{Definition}[section]
\newtheorem{problem}{\em Problem}[section]
\newtheorem{remark}{Remark}[section]
\newtheorem{example}{Example}[section]
\newtheorem{case}{Case}[section]
\newtheorem{assumption}{Assumption}[section]

\renewcommand\proofname{\bf Proof} 
\def\eop{$\rule{1.3ex}{1.3ex}$}
\renewcommand\qedsymbol\eop  
\numberwithin{equation}{section}
\makeatletter

\newcommand{\XXX}{{\bf XXX~}}
\newcommand{\beq}{\begin{equation}} \newcommand{\eeq}{\end{equation}}
\newcommand{\bz}{{\bf z}}
\newcommand{\bx}{{\bf x}}
\newcommand{\bt}{{\bf t}}
\newcommand{\bi}{\begin{itemize}}
\newcommand{\be}{\begin{enumerate}}
\newcommand{\ei}{\end{itemize}}
\newcommand{\ee}{\end{enumerate}}
\newcommand{\calH}{{\cal H}}
\newcommand{\E}{{\cal E}}
\newcommand{\Om}{{\Theta}}
\newcommand{\om}{{\theta}}
\newcommand{\gam}{{\gamma}}
\newcommand{\Gam}{{\Gamma}}
\newcommand{\homega}{{\hat \om}}
\newcommand{\bS}{{\mathbb S}}
\newcommand{\R}{{\mathbb R}}
\newcommand{\N}{{\mathbb N}}
\newcommand{\calE}{{\cal E}}
\newcommand{\calG}{{\cal G}}
\newcommand{\calV}{{\cal V}}
\newcommand{\calK}{{\cal K}}
\newcommand{\calN}{{\cal N}}
\newcommand{\calU}{{\cal U}}
\newcommand{\calT}{{\cal T}}
\newcommand{\calY}{{\cal Y}}
\newcommand{\calO}{{\cal O}}
\newcommand{\calX}{{\cal X}}
\newcommand{\calW}{{\cal W}}
\newcommand{\W}{{\cal W}}
\newcommand{\G}{{\cal G}}
\newcommand{\K}{{\cal K}}
\newcommand{\OO}{{\bf O}}
\newcommand{\hK}{{\hat K}}
\newcommand{\X}{{\cal X}}
\newcommand{\M}{{\cal M}}
\newcommand{\KG}{{\cal K(\calG)}}
\newcommand{\lam}{{\lambda}}
\newcommand{\calM}{{\cal M}}
\newcommand{\calA}{{\cal A}}
\newcommand{\calB}{{\cal B}}
\newcommand{\calL}{{\cal L}}
\newcommand{\calD}{{\cal D}}
\newcommand{\calR}{{\cal R}}
\newcommand{\pp}{{\cal P}}
\newcommand{\hc}{{\hat c}}
\newcommand{\ck}{{c_K}}
\newcommand{\hL}{{\hat L}}
\newcommand{\tL}{{\bar L}}
\newcommand{\sK}{{\SSS K}}
\newcommand{\hg}{{g_K}}
\newcommand{\tf}{{f_K}}
\newcommand{\hy}{{y_K}}
\newcommand{\haty}{{\hat y}}
\newcommand{\hG}{{\hat \Gam}}
\newcommand{\vt}{{\vec t}}
\newcommand{\vv}{{\vec v}}
\newcommand{\lb}{{\langle}}
\newcommand{\rb}{{\rangle}}
\newcommand{\by}{{\bf y}}
\newcommand{\btau}{{\bf \tau}}
\newcommand{\bu}{{\bf u}}
\newcommand{\bv}{{\bf v}}
\newcommand{\tby}{\tilde{{\bf y}}}
\newcommand{\Sb}{{\bf S}}
\newcommand{\Mb}{{\bf M}}
\newcommand{\Ob}{{\bf O}}
\newcommand{\SSS}{\scriptscriptstyle}
\def\boldf#1{\hbox{\rlap{$#1$}\kern.4pt{$#1$}}}
\newcommand{\balpha}{{\boldf \alpha}}
\newcommand{\wh}{\hat w}
\newcommand{\Wh}{\hat W}
\newcommand{\wb}{\bar w}
\newcommand{\Wb}{\bar W}
\newcommand{\xb}{\bar x}
\newcommand{\cb}{\bar c}
\newcommand{\trans}{^{\scriptscriptstyle \top}}
\newcommand{\tW}{\tilde{W}}
\newcommand{\tw}{\tilde{w}}

\newcommand{\figsheight}{4.0cm}
\renewcommand\baselinestretch{1}

\mathtoolsset{showonlyrefs,showmanualtags}

\begin{titlepage}
\advance\topmargin by 0.5in
\begin{center}

\vspace{1.5truecm} {\Large When is there a representer theorem? \\ 
Vector versus matrix regularizers
}

\vspace{1.2truecm}

\end{center}

\begin{center}

{\bf Andreas Argyriou}$^{(1)}$, {\bf Charles A. Micchelli}$^{(2)}$,
{\bf Massimiliano Pontil}$^{(1)}$

\vspace{.75truecm} \vspace{.75truecm}

(1) Department of Computer Science \\
University College London \\
Gower Street, London WC1E, England, UK \\
E-mail: {\em m.pontil@cs.ucl.ac.uk} \vspace{.75truecm}

\noindent (2) Department of Mathematics and Statistics\\
State University of New York \\
The University at Albany \\
1400 Washington Avenue, Albany, NY, 12222, USA \\
E-mail: {\em cam@math.albany.edu }

\vspace{.75truecm}

\today

\end{center}

\vspace{.5truecm}

\begin{abstract}
{\noindent We consider a general class of regularization methods which
learn a vector of parameters on the basis of linear measurements. It
is well known that if the regularizer is a nondecreasing function of
the inner product then the learned vector is a linear combination of
the input data. This result, known as the {\em representer theorem}, is at
the basis of kernel-based methods in machine learning. In this paper,
we prove the necessity of the above condition, thereby completing the
characterization of kernel methods based on regularization. We further
extend our analysis to regularization methods which learn a matrix, a
problem which is motivated by the application to multi-task
learning. In this context, we study a more general representer
theorem, which holds for a larger class of regularizers. We provide a
necessary and sufficient condition for these class of matrix
regularizers and highlight them with some concrete examples of
practical importance. Our analysis uses basic principles from matrix
theory, especially the useful notion of matrix nondecreasing function.
}
\end{abstract}
\end{titlepage}



\section{Introduction}
\label{sec:intro}

Regularization in Hilbert spaces is an important methodology for
learning from examples and has a long history in a variety of
fields. It has been studied, from different perspectives, in
statistics \citep{wahba}, in optimal estimation \citep{micchelli} and
recently has been a focus of attention in machine learning theory --
see, for example, \citep{cucker,Devito,mp_jmlr,JST,vapnik} and
references therein. Regularization is formulated as an {\em
optimization problem} involving an {\em error term} and a {\em
regularizer}. The regularizer plays an important role, in that it
favors solutions with certain desirable properties. It has long been
observed that certain regularizers exhibit an appealing property,
called the {\em representer theorem}, which states that there exists a
solution of the regularization problem that is a linear combination of
the data \citep{wahba}. This property has important computational
implications in the context of regularization with positive
semidefinite {\em kernels}, because it makes high or
infinite-dimensional problems of this type into finite dimensional
problems of the size of the number of available data
\citep{scholkopf,JST}.

The topic of interest in this paper will be to determine the
conditions under which representer theorems hold. In the first half of
the paper, we describe a property which a regularizer should satisfy
in order to give rise to a representer theorem. It turns out that this
property has a simple geometric interpretation and that the
regularizer can be equivalently expressed as a {\em nondecreasing}
function of the Hilbert space norm. Thus, we show that this condition,
which has already been known to be sufficient for representer
theorems, is also {\em necessary}.  In the second half of the paper,
we depart from the context of Hilbert spaces and focus on a class of
problems in which a {\em matrix structure} plays an important role. For
such problems, which have recently appeared in several machine
learning applications, we show a modified version of the representer
theorem that holds for a class of regularizers significantly larger
than in the former context.  As we shall see, these matrix
regularizers are important in the context of multi-task learning: the
matrix columns are the parameters of different regression tasks and
the regularizer encourages certain dependences across the tasks.

\bigskip

In general, we consider problems in the framework of {\em Tikhonov
regularization} \citep{tikhonov}. This regularization approach
receives a set of input/output data $(x_1,y_1),\dots,$ $(x_m,y_m) \in
\calH \times \calY$ and selects a vector in $\calH$ as the solution of an
optimization problem. Here, $\calH$ is a prescribed Hilbert space
equipped with the inner product $\lb \cdot, \cdot \rb$ and $\calY
\subseteq \R$ a set of possible output values. The optimization
problems encountered in regularization are of the type
\beq
\min\left\{ \E \bigl( \left( \lb w,x_1\rb,\dots,\lb w,x_m \rb \right) ,
\left( y_1, \dots, y_m \right) \bigr) +
\gamma \, \Omega(w): w \in \calH \right\} \,,
\label{eq:reg_intro}
\eeq 
where $\gamma > 0$ is a regularization parameter. The function
$\E:\R^m \times \calY^m \rightarrow \R$ is called an {\em error
function} and $\Omega: \calH \rightarrow \R$ is called a {\em
regularizer}.  The error function measures the error on the
data. Typically, it decomposes as a sum of univariate functions. For
example, in regression, a common choice would be the sum of square
errors, $\sum_{i=1}^m (\lb w,x_i\rb-y_i)^2$. The function $\Omega$,
called the regularizer, favors certain regularity properties of the
vector $w$ (such as a small norm) and can be chosen based on available
prior information about the target vector. In some
Hilbert spaces such as Sobolev spaces the regularizer is measure of
smoothness: the smaller the norm the smoother the function.

This framework includes several well-studied learning algorithms, such
as ridge regression \citep{Hoerl}, support vector machines
\citep{Guyon}, and many more -- see \citep{scholkopf,JST} and references therein.

An important aspect of the practical success of this approach is the
observation that, for certain choices of the regularizer, solving
\eqref{eq:reg_intro} reduces to identifying $m$ parameters and not
$\dim(\calH)$. Specifically, when the regularizer is the square of the
Hilbert space norm, the representer theorem holds: there
exists a solution ${\hat w}$ of \eqref{eq:reg_intro} which is a linear
combination of the input vectors,
\beq 
{\hat w} = \sum_{i=1}^m c_i x_i,
\label{eq:RT} 
\eeq
where $c_i$ are some real coefficients. This result is simple to prove
and dates at least from the 1970's, see, for example, \citep{KW}. It is also known that it
extends to any regularizer that is a {\em nondecreasing} function of
the norm \citep{SHS}. Several other variants and results about the
representation form \eqref{eq:RT} have also appeared in recent years
\citep{Devito,Nicolao,EPP,GJP,vector,steinwart}. Moreover, the
representer theorem has been important in machine learning,
particularly within the context of learning in reproducing kernel
Hilbert spaces \citep{arons} -- see \citep{scholkopf,JST} and
references therein.
 
\bigskip
Our first objective in this paper is to derive necessary
and sufficient conditions for representer theorems to hold. Even
though one is mainly interested in regularization problems, it is more
convenient to study {\em interpolation} problems, that is,
problems of the form
\beq 
\min\left\{ \Omega(w): w \in \calH ,
\lb w,x_i\rb = y_i,~\forall i=1,\dots,m
\right\} \,.
\label{eq:int_intro}
\eeq 
Thus, we begin this paper (Section \ref{sec:reg_int}) by showing how
representer theorems for interpolation and regularization relate.  On
one side, a representer theorem for interpolation easily implies
such a theorem for regularization with the same regularizer and any
error function. Therefore, {\em all representer theorems obtained in
this paper apply equally to interpolation and regularization}.  On the
other side, though, the converse implication is true under
certain weak qualifications on the error function.

Having addressed this issue, we concentrate in Section
\ref{sec:vector} on proving that an interpolation 
problem \eqref{eq:int_intro} admits
solutions representable in the form \eqref{eq:RT} {\em if and only if}
the regularizer is {\em a nondecreasing function of the Hilbert space
norm}. That is, we provide a complete characterization of regularizers
that give rise to representer theorems, which had been an open
question. Furthermore, we discuss how our proof is motivated by a
geometric understanding of the representer theorem, which is
equivalently expressed as a monotonicity property of the regularizer.

\bigskip 
Our second objective is to formulate and study the novel question of
representer theorems for {\em matrix problems}. To make our discussion concrete, let us consider the problem of learning $n$ linear
regression vectors, represented by the parameters $w_1,\dots,w_n \in
\R^d$, respectively. Each vector can be thought of as a
``task'' and the goal is to {\em jointly} learn these $n$ tasks. In
such problems, there is usually prior knowledge that {\em relates}
these tasks and it is often the case that learning can improve if this
knowledge is appropriately taken into account. Consequently, a good
regularizer should favor such task relations and involve {\em all
tasks jointly}.

In the case of interpolation, this learning framework can be
formulated concisely as
\begin{equation}
\min \left\{ \Omega(W) : W \in \Mb_{d,n} \, , ~ w_t\trans x_{ti} = y_{ti} ~~ \forall i=1,\dots,m_t, ~ t=1,\dots,n
\right\} \,,
\label{eq:matrix_intro}
\end{equation}
where $\Mb_{d,n}$ denotes the set of $d\times n$ real matrices and the
column vectors $w_1,\dots,w_n \in \R^d$ form the matrix $W$. Each task 
$t$ has its own input data $x_{t1},\dots,x_{t m_t} \in \R^d$ and corresponding 
output values $y_{t1},\dots,y_{t m_t} \in \calY$. 


An important feature of such problems that distinguishes them from the
type \eqref{eq:int_intro} is the appearance of {\em matrix
products} in the constraints, unlike the inner products in
\eqref{eq:int_intro}. In fact, as we will discuss in Section 
\ref{sec:matrix_intro}, problems of the type \eqref{eq:matrix_intro}
can be written in the form \eqref{eq:int_intro}. Consequently, the
representer theorem applies if the matrix regularizer is a
nondecreasing function of the Frobenius norm\footnote{Defined as
$\|W\|_2 = \sqrt{\trace(W\trans W)}$.}. However, the optimal vector
${\hat w}_t$ for each task can be represented as a linear combination
of {\em only those input vectors corresponding to this particular
task}. Moreover, with such regularizers it is easy to see that each
task in \eqref{eq:matrix_intro} can be optimized independently. Hence,
these regularizers are of no practical interest if the tasks are
expected to be related.

This observation leads us to formulate a {\em modified representer theorem},
which is appropriate for matrix problems, namely,
\begin{align}
\label{eq:rep_matrix_intro}
&&&& {\hat w}_t = \sum_{s=1}^n \sum_{i=1}^{m_s} c^{(t)}_{si} x_{si} && \forall \, t=1,\dots,n ,
\end{align}
where $c^{(t)}_{si}$ are scalar coefficients, for $t,s=1,\dots,n, ~ i
=1,\dots,m_s$.  In other words, we now allow for {\em all input
vectors} to be present in the linear combination representing each column
of the optimal matrix. As a result, this definition greatly expands the class 
of regularizers that give rise to representer theorems.

Moreover, this framework can be applied to many applications where
matrix optimization problems are involved. Our immediate motivation, however, 
has been more specific than that, namely {\em multi-task learning}.   
Learning multiple tasks
jointly has been a growing area of interest in machine learning,
especially during the past few years
\citep{bach_theo,mtl_feat,ml,spectral,candes,CCG_08,Izenman,Maurer2,maurer,Srebro,Wolf,bennett_mtl,ming}.
For instance, some of these works use regularizers which involve
the {\em trace norm}\footnote{Equal to the sum of the
singular values of $W$.} of matrix $W$. The general idea behind this methodology is
that a small trace norm favors low-rank matrices. This means that the
tasks (the columns of $W$) are related in that they all lie in a
low-dimensional subspace of $\R^d$. In the case of the trace norm,
the representer theorem \eqref{eq:rep_matrix_intro} is known to hold
-- see \citep{bach_theo,ml,srebro_icml}, also discussed in Section \ref{sec:matrix_intro}.

 
It is natural, therefore, to ask a question similar to that in the standard
Hilbert space (or single-task) setting. That is, under which
conditions on the regularizer a representer theorem holds. In Section
\ref{sec:matrix_proof}, we provide an answer by {\em proving a
necessary and sufficient condition for representer theorems to hold,
expressed as a simple monotonicity property}. This property is
analogous to the one in the Hilbert space setting, but its geometric
interpretation is now algebraic in nature.
We also give a functional description equivalent to this
property, that is, {\em we show that the regularizers of interest are
the matrix nondecreasing functions of the quantity $W\trans W$}.

Our results cover matrix problems of the type
\eqref{eq:matrix_intro} which have already been studied in the
literature. But they also point towards some new learning methods that
may perform well in practice and can now be made computationally
efficient. Thus, we close the paper with a discussion of possible
regularizers that satisfy our conditions and have been used or can be
used in the future in machine learning problems.

\subsection{Notation}

Before proceeding, we introduce the notation used in this paper. We
use $\N_d$ as a shorthand for the set of integers $\{1,\dots,d\}$.  We
use $\R^d$ to denote the linear space of vectors with $d$ real
components. The standard inner product in this space is denoted by
$\lb\cdot,\cdot\rb$, that is, $\lb w,v\rb = \sum_{i\in\N_d} w_i v_i,
~\forall w,v \in \R^d$, where $w_i,v_i$ are the $i$-th components of
$w,v$ respectively. More generally, we will consider Hilbert
spaces which we will denote by $\calH$, equipped with an inner product
$\lb\cdot,\cdot\rb$.

We also let $\Mb_{d,n}$ be the linear space of $d \times n$ real
matrices. If $W,Z \in \Mb_{d,n}$ we define their Frobenius inner
product as $\lb W,Z\rb = \trace(W\trans Z)$, where $\trace$ denotes
the trace of a matrix. With $\Sb^d$ we denote the set of $d\times d$
real symmetric matrices and with $\Sb^d_+$ ($\Sb^d_{++}$) its subset
of positive semidefinite (definite) ones. We use $\succ$ and $\succeq$
for the positive definite and positive semidefinite partial orderings,
respectively. Finally, we let $\OO^d$ be the set of $d \times d$
orthogonal matrices.


\section{Regularization versus Interpolation}
\label{sec:reg_int}

The line of attack we shall follow in this paper will go through {\em
interpolation}. That is, our main concern will be to obtain necessary
and sufficient conditions for representer theorems that hold for
interpolation problems. However, in practical applications one
encounters {\em regularization} problems more frequently than
interpolation problems.

First of all, the family of the former problems is more general than
that of the latter ones. Indeed, an interpolation problem can be
simply obtained in the limit as the {\em regularization parameter}
goes to zero \citep{micchelli-pinkus}. 
More importantly, regularization enables one to
trade off interpolation of the data against smoothness or simplicity
of the model, whereas interpolation frequently suffers from {\em
overfitting}. 

Thus, frequently one considers problems of the form
\beq
\min\left\{ \E \bigl( \left( \lb w,x_1\rb,\dots,\lb w,x_m \rb \right) ,
\left( y_1, \dots, y_m \right) \bigr) +
\gamma \, \Omega(w): w \in \calH \right\} \,,
\label{eq:reg_general}
\eeq 
where $\gamma > 0$ is called the regularization parameter. This parameter is not known in
advance but can be tuned with techniques like {\em cross validation}
\citep{wahba}.  Here, $\Omega: \calH \rightarrow \R$ is a {\em
regularizer}, $\E: \R^m \times \calY^m \rightarrow \R$ is an error
function and $x_i \in \calH, y_i \in \calY,
\forall i \in \N_m,$ are given input and output data.
The set $\calY$ is a subset of $\R$ and varies depending on the context, so that
it is typically assumed equal to $\R$ in the case
of regression or equal to $\{-1,1\}$ in binary classification problems.
One may also consider the associated interpolation problem, which is
\beq 
\min\left\{ \Omega(w): w \in \calH ,
\lb w,x_i\rb = y_i,~\forall i \in \N_m
\right\} \,.
\label{eq:int_general}
\eeq 

Under certain assumptions, the minima in problems
\eqref{eq:reg_general} and \eqref{eq:int_general} are attained
(whenever the constraints in \eqref{eq:int_general} are satisfiable).
Such assumptions could involve, for example, lower
semi-continuity and boundedness of sublevel sets for $\Omega$ and 
boundedness from below for $\E$. 
These issues will not concern us here, as we shall assume the following 
about the error function $\E$ and the regularizer $\Omega$, from now on.
\begin{assumption}
The minimum \eqref{eq:reg_general} is attained for any $\gamma > 0$,
any input and output data $\{x_i, y_i : i \in \N_m \}$ and any $m \in
\N$.  The minimum \eqref{eq:int_general} is attained for any input and
output data $\{x_i, y_i : i \in \N_m \}$ and any $m \in \N$, whenever
the constraints in \eqref{eq:int_general} are satisfiable..
\label{assumption}
\end{assumption} 


The main objective of this paper is to obtain {\em necessary and sufficient}
conditions on $\Omega$ so that the solution of problem
\eqref{eq:reg_general} satisfies a {\em linear representer theorem}.

\begin{definition}
We say that a class of optimization problems such as
\eqref{eq:reg_general} or \eqref{eq:int_general} satisfies the {\em
linear representer theorem} if, for any choice of data $\{ x_i, y_i :
i \in \N_m \}$ such that the problem has a solution, {\em there
exists} a solution that belongs to $\Span \{ x_i : i \in \N_m \}$.
\label{def:rep}
\end{definition}

In this section, we show that the existence of representer theorems
for regularization problems is equivalent to the existence of
representer theorems for interpolation problems, under a quite general
condition that has a rather simple geometric interpretation.

We first recall a lemma from \citep[Sec. 2]{banach} which states that 
(linear or not) representer theorems for interpolation lead to 
representer theorems for regularization, under no conditions on the error function.

\begin{lemma}
Let $\E: \R^m \times \calY^m \rightarrow \R$, $\Omega : \calH \rightarrow \R$ satisfying Assumption
\ref{assumption}.
Then if the class of interpolation problems \eqref{eq:int_general} satisfies the linear representer theorem,
so does the class of regularization problems \eqref{eq:reg_general}.
\label{lem:reg_int}
\end{lemma}

\begin{proof}
Consider a problem of the form \eqref{eq:reg_general} and let $\wh$ be a solution.
We construct an associated interpolation problem
\beq
\min\left\{ \Omega(w): w \in \calH ,
\lb w,x_1\rb = \lb \wh, x_1 \rb,\dots,\lb w,x_m \rb = \lb \wh, x_m \rb
\right\} \,.
\label{eq:aux_int}
\eeq
By hypothesis, there exists a solution $\tw$ of \eqref{eq:aux_int} that lies in
$\Span \{x_i : i \in \N_m\}$. But then $\Omega(\tw) \leq \Omega(\wh)$ and 
hence $\tw$ is a solution of \eqref{eq:reg_general} and the result follows.
\end{proof}


This lemma requires no special properties of the functions
involved. Its converse, in contrast, requires assumptions about the
analytical properties of the error function. We provide one such
natural condition in the theorem below, but other conditions could
conceivably work too. The main idea in the proof is, based on a single
input, to construct a sequence of appropriate regularization problems
for different values of the regularization parameter $\gamma$. Then,
it suffices to show that letting $\gamma
\to 0^+$ yields a limit of the minimizers that satisfies an interpolation constraint.

\begin{theorem}
Let $\E: \R^m \times \calY^m \rightarrow \R$ 
and $\Omega : {\cal H} \to \R$.
Assume that $\E, \Omega$ are lower semi-continuous,
that $\Omega$ has bounded sublevel sets and that
$\E$ is bounded from below.
Assume also that, for some $v \in \R^m \setminus \{0\}, y \in \calY^m$,
there exists a {\em unique} minimizer of
$\min\{\E(av,y) : a \in \R \}$ and that this minimizer does not equal zero.
Then if the class of regularization problems \eqref{eq:reg_general}
satisfies the linear representer theorem, so does 
the class of interpolation problems \eqref{eq:int_general}.
\label{thm:reg_int}
\end{theorem}

\begin{proof}
Fix an arbitrary $x\neq 0$ and let $a_0$ be the minimizer of
$\min\{\E(av,y) : a \in \R \}$.
Consider the problems
$$
\min\left\{ \E \left(\dfrac{a_0}{\|x\|^2}\lb w, x\rb\, v, y \right) + 
\gamma\, \Omega(w) : w \in \calH \right \} ,
$$
for every $\gamma > 0$, 
and let $w_\gamma$ be a solution in the span of $x$ (known to exist by hypothesis).
We then obtain that
\begin{align}
\E(a_0v,y) + \gamma \Omega(w_\gamma) & \leq 
\E \left( \dfrac{a_0}{\|x\|^2} \lb w_\gamma, x\rb \, v, y \right) + 
\gamma\, \Omega(w_\gamma) \leq
\E \left(a_0\, v, y \right) + \gamma\, \Omega\left(x \right) \,.
\label{eq:reg_gamma}
\end{align}
Thus, $\Omega(w_\gamma) \leq \Omega\left(x \right)$
and so, by the hypothesis on $\Omega$, the set $\{w_\gamma: \gamma > 0\}$ is bounded. Therefore, there exists a
convergent subsequence $\{w_{\gamma_\ell} : \ell \in \N\}$, with $\gamma_\ell \to 0^+$,
whose limit we call $\wb$.
By taking the limits as $\ell \to \infty$ on the inequality on the right
in \eqref{eq:reg_gamma}, we obtain
\begin{align}
\E \left( \dfrac{a_0}{\|x\|^2}\lb \wb, x\rb\,v, y \right)
\leq \E \left(a_0\,v, y \right) 
\end{align}
and consequently
\begin{align}
\dfrac{a_0}{\|x\|^2}\lb \wb, x\rb\ = a_0 
\end{align}
or 
$$
\lb \wb, x\rb\ = \|x\|^2.
$$

In addition, since $w_\gamma$ belongs to the span of $x$ for every $\gamma > 0$, 
so does $\wb$.  
Thus, we obtain that $\wb = x$.
Moreover, from the definition of $w_\gamma$ we have that
\begin{flalign}
&& \E \left( \dfrac{a_0}{\|x\|^2} \lb w_\gamma, x\rb \, v, y \right) + 
\gamma\, \Omega(w_\gamma) \leq
\E \left( a_0 \, v, y \right) + 
\gamma\, \Omega(w) 
&&  \forall w\in\calH ~\text{such that}~ \lb w,x\rb=\|x\|^2
\end{flalign}
and, combining with the definition of $a_0$, that
\begin{flalign}
&& \Omega(w_\gamma) \leq
\Omega(w) .
&&  \forall w\in\calH ~\text{such that}~ \lb w,x\rb=\|x\|^2
\end{flalign}
Taking the limits as $\ell \to \infty$, we conclude that
$\wb=x$ is a solution of the problem
$$
\min\{\Omega(w) : w \in \calH, \lb w,x\rb = \|x\|^2\} \,.
$$
Moreover, this assertion holds even when $x=0$, since the hypothesis
implies that $0$ is a global minimizer of $\Omega$.
Indeed, any regularization problem of the type \eqref{eq:reg_general} with zero inputs,
$x_i = 0, \forall i\in\N_m$, admits a solution in their span.
Thus, we have shown that $\Omega$ satisfies property \eqref{eq:geom}
and the result follows immediately from Lemma \ref{lem:1case}.
\end{proof}

We now comment on some commonly used error functions.
The first is the {\em square loss}, 
$$
\E(z,y) = \sum_{i\in\N_m} (z_i-y_i)^2 \,,
$$
for $z,y \in \R^m$. 
It is immediately apparent that Theorem \ref{thm:reg_int} applies in this case.

The second case is the
{\em hinge loss}, 
$$
\E(z,y) = \sum_{i\in\N_m} \max(1-z_i y_i,0) \,,
$$ 
where the outputs $y_i$ are assumed to belong to $\{-1,1\}$ for the purpose of
classification. 
In this case, we may select
$y_i = 1,\forall i\in\N_m,$ and $v = (-1,-2,0,\dots,0)\trans$ for $m\geq 2$.
Then the function $\E(\cdot\,v,y)$ is the one shown in
Figure \ref{fig:hinge}.
\begin{figure}
\begin{center}
\includegraphics[width=0.5\textwidth]{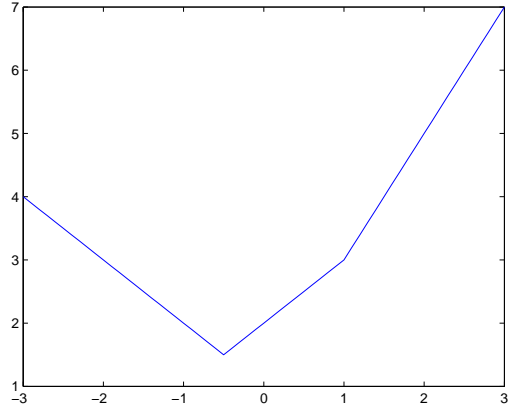}
\caption{Hinge loss along the direction $(1,-2,0,\dots,0)$.}
\label{fig:hinge}
\end{center}
\end{figure}

Finally, the {\em logistic loss}, 
$$
\E(z,y) = 
\sum_{i\in\N_m} \log \left( 1+ e^{-z_i y_i} \right) \,,
$$ is also used in classification problems.  In this case, we may
select $y_i = 1,\forall i\in\N_m,$ and $v = (2,-1)\trans$ for $m=2$ or
$v = (m-2,-1,\dots,-1)\trans$ for $m>2$. In the latter case, for
example, setting to zero the derivative of $\E(\cdot\, v,y)$ yields
the equation $(m-1) e^{a(m-1)}+e^a-m+2 =0$, which can easily be seen
to have a unique solution.

Summarizing, we obtain the following corollary.
\begin{corollary}
If $\E: \R^m \times \calY^m \rightarrow \R$ is the square loss,
the hinge loss (for $m\geq2$) or the logistic loss (for $m\geq2$)
and $\Omega : \calH \rightarrow \R$ is lower semi-continuous with bounded sublevel sets, 
then the class of problems
\eqref{eq:reg_general} satisfies the linear representer theorem if and only if
the class of problems \eqref{eq:int_general} does.
\end{corollary}

Note also that the condition on $\E$ in Theorem \ref{thm:reg_int} is rather weak
in that an error function $\E$ may satisfy it without being convex.
At the same time, an error function that is ``too flat'', such as
a constant loss, will not do.

We conclude with a remark about the situation in which
the inputs $x_i$ are {\em linearly independent}.\footnote{This occurs frequently
in practice, especially when the dimensionality $d$ is high.}
It has a brief and straightforward proof, which we do not present here.

\begin{remark}
Let $\E$ be the hinge loss or the logistic loss and $\Omega : {\cal H} \to
\R$ be of the form $\Omega(w) = h(\|w\|)$, where $h:\R_+\to \R$ is a
lower semi-continuous function with bounded sublevel sets. Then the
class of regularization problems \eqref{eq:reg_general} in which the
inputs $x_i,i\in\N_m,$ are linearly independent, satisfies the
linear representer theorem.
\end{remark}


\section{Representer Theorems for Interpolation Problems}
\label{sec:vector}

The results of the previous section allow us
to focus on linear representer theorems for interpolation problems of the type
\eqref{eq:int_general}. We are going to consider the case of a Hilbert space
$\calH$ as the domain of an interpolation problem. Interpolation
constraints will be formed as inner products of the variable with the
input data. For all purposes in this context, it makes no difference
to think of $\calH$ as being equal to $\R^d$.

In this section, we consider the interpolation problem
\begin{equation}
\min\{\Omega(w): w \in \calH, \lb w,x_i \rb = y_i, i \in \N_m\},
\label{eq:min-int}
\end{equation}
We coin the term {\em admissible} to denote the class of regularizers
we are interested in.
\begin{definition}
We say that the function $\Omega: {\cal H} \rightarrow \R$ is {\em
admissible} if, for every $m \in \N$ and any data set $\{(x_i,y_i): i \in
\N_m\} \subseteq {\cal H} \times \calY$ such that the interpolation
constraints are satisfiable, problem \eqref{eq:min-int} admits a solution
${\hat w}$ of the form
\begin{equation}
{\hat w} = \sum_{i\in\N_m} c_i x_i,
\label{eq:rep}
\end{equation}
where $c_i$ are some real parameters.
\label{def:admissible}
\end{definition}
We say that $\Omega: {\cal H} \rightarrow \R$ is differentiable if,
for every $w \in {\cal H}$, there is a unique vector denoted by
$\nabla \Omega(w)$, such that for all $p \in {\cal H}$, $$
\lim_{t \rightarrow 0} \frac{\Omega(w+ t p)-\Omega(w)}{t} = 
\lb \nabla \Omega(w),p\rb.
$$ This notion corresponds to the usual notion of directional
derivative on $\R^d$ and in that case $\nabla
\Omega(w)$ is the gradient of $\Omega$ at $w$.

In the remainder of the section, we always assume that Assumption \ref{assumption}
holds for $\Omega$.
The following theorem provides a necessary and sufficient condition
for a regularizer to be admissible. 
\begin{theorem}
\label{thm:main1}
Let $\Omega: \calH \rightarrow \R$ be a differentiable function and
${\rm dim} (\calH) \geq 2$. Then $\Omega$ is admissible if and only if
\begin{align}
&&&& \Omega(w) = h(\lb w, w\rb) &&&& \forall~w \in \calH,
\label{eq:functional_vector}
\end{align}
for some nondecreasing function $h : \R_+ \rightarrow \R$.
\end{theorem}

It is well known that the above functional form is sufficient for a
representer theorem to hold (see for example \citep{SHS}). Here we show
that it is also necessary.

The route we follow to proving the above theorem is based on a
geometric interpretation of representer theorems. This intuition can
be formally expressed as condition \eqref{eq:geom} in the lemma
below. Both condition \eqref{eq:geom} and functional
form \eqref{eq:functional_vector} express the property that the
contours of $\Omega$ are {\em spheres} (or regions between spheres), 
which is apparent from Figure \ref{fig:contours}.

\begin{figure}
\begin{center}
\includegraphics[width=0.6\textwidth]{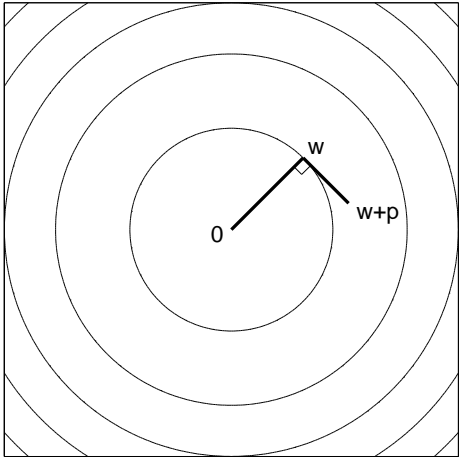}
\end{center}
\caption{Geometric interpretation of Theorem \ref{thm:main1}. The function $\Omega$ should not decrease when moving to
orthogonal directions. The contours of such a function should be spherical.}
\label{fig:contours}
\end{figure}

\begin{lemma}
A function $\Omega: \calH \rightarrow \R$ is admissible if and only if it satisfies the property
that 
\begin{align}
&&&& \Omega(w+p) \geq \Omega(w) && \forall~ w,p \in \calH ~\text{such that}~ \lb w,p\rb= 0.
\label{eq:geom}
\end{align}
\label{lem:1case}
\end{lemma}
\begin{proof}
Suppose that $\Omega$ satisfies property (\ref{eq:geom}),
consider arbitrary data $x_i,y_i, i \in \N_m,$
and let $\wh$ be a solution to problem \eqref{eq:min-int}. We can
uniquely decompose $\wh$ as $\wh = \wb + p$ where
$\wb \in \calL := \Span\{x_i : i \in \N_m\}$ and $p \in \calL^\perp$.
From \eqref{eq:geom} we obtain that $\Omega(\wh) \geq \Omega(\wb)$. 
Also $\wb$ satisfies the interpolation constraints
and hence we conclude that $\wb$ is a solution to problem \eqref{eq:min-int}.

Conversely, if $\Omega$ is admissible choose any $w \in {\cal H}$ and
consider the problem $\min \{ \Omega(z) : z \in \calH, \lb z,w \rb = \lb w,w\rb\}$. By
hypothesis, there exists a solution belonging in $\Span\{w\}$ and hence 
$w$ is a solution to this problem. Thus, we have that 
$\Omega(w+p) \geq \Omega(w)$ for every $p$ such that $\lb w, p \rb=0$.
\end{proof}


It remains to establish the equivalence of the geometric property \eqref{eq:geom} 
to condition \eqref{eq:functional_vector} that $\Omega$ is a nondecreasing function of
the $L_2$ norm.

\begin{proof}[Proof of Theorem \ref{thm:main1}]
Assume first that \eqref{eq:geom} holds and ${\rm dim} (\calH) <
\infty$. In this case, we only need to consider the case that $\calH =
\R^d$ since \eqref{eq:geom} can always be rewritten as an equivalent
condition on $\R^d$, using an orthonormal basis of $\calH$.

First we observe that, since $\Omega$ is differentiable, this property
implies the condition that 
\beq
\lb \nabla \Omega(w),p \rb = 0 \,,
\label{eq:grad_zero}
\eeq
for all $w,p \in \R^d$ such that $\lb w,p\rb = 0$. 

Now, fix any $w_0 \in \R^d$ such that $\|w_0\| = 1$.  Consider an
arbitrary $w \in \R^d$. Then there exists an orthogonal matrix $U \in
\OO^d$ such that $w = \|w\| U w_0$ and $\det(U)=1$ (see Lemma
\ref{lem:3} in the appendix). Moreover, we can write $U=e^D$ for some skew-symmetric
matrix $D\in\Mb_{d,d}$ --- see \citep[Example 6.2.15]{horn2}.  Consider
now the path $z:[0,1] \to \R^d$ with
\begin{align}
&&&& z(\lambda) = \|w\| e^{\lambda D} w_0 && \forall~\lambda \in [0,1].
\end{align}
We have that $z(0) = \|w\| w_0$ and $z(1) = w$.
Moreover, since $ \lb z(\lambda),z(\lambda)\rb = \lb w,w\rb$,
we obtain that
\begin{align}
&&&& \lb z'(\lambda),z(\lambda)\rb = 0 && \forall~\lambda \in [0,1].
\end{align}
Applying \eqref{eq:grad_zero} with $w = z(\lambda),p = z'(\lambda)$, it follows
that
$$ 
\frac{d \Omega(z(\lambda))}{d \lambda} = \lb \nabla
\Omega(z(\lambda)),z'(\lambda) \rb = 0.
$$ 
Consequently, $\Omega(z(\lambda))$ is constant and hence
$\Omega(\|w\| w_0) = \Omega(w)$. Setting $h(\xi) = \Omega(\sqrt{\xi}
w_0),~\forall \xi \in \R_+,$ yields \eqref{eq:functional_vector}. In
addition, $h$ must be nondecreasing in order for $\Omega$ to satisfy
property (\ref{eq:geom}).

For the case ${\rm dim} (\calH) = \infty$ we can argue similarly using
instead the path $$ z(\lambda) = \frac{(1-\lambda) w_0 + \lambda
w}{\|(1-\lambda) w_0 + \lambda w\|} \|w\| $$ which is differentiable
on $[0,1]$ when $w \notin {\rm span}\{w_0\}$. We confirm equation
\eqref{eq:functional_vector} for vectors 
in ${\rm span}\{w_0\}$ by a limiting argument on  vectors not in ${\rm
span}\{w_0\}$ since $\Omega$ is surely continuous.

Conversely, if $\Omega(w) = h(\lb w, w\rb)$ and $h$ is nondecreasing,
property (\ref{eq:geom}) follows immediately.
\end{proof}


We note that we could modify Definition \ref{def:admissible} by requiring
that {\em any} solution of problem \eqref{eq:min-int} be in the linear
span of the input data. We call such regularizers {\em strictly admissible}. 
Then with minor modifications to Lemma
\ref{lem:1case} (namely, requiring that equality in
\eqref{eq:geom} holds only if $p =0$) and to the proof of Theorem \ref{thm:main1}
(namely, requiring $h$ to be strictly increasing) we have the
following corollary.

\begin{corollary}
\label{cor:main1}
Let $\Omega: \calH \rightarrow \R$ be a differentiable function. 
Then $\Omega$ is {\em strictly} admissible if and only if $\Omega(w) =
h(\lb w, w\rb), ~\forall w\in\calH$, where $h : \R_+ \rightarrow \R$ is {\em strictly} increasing.
\end{corollary}

Theorem \ref{thm:main1} can be used to verify whether the linear
representer theorem can be obtained when using a regularizer $\Omega$.
For example, the function $\|w\|_p = (\sum_{i\in\N_d}
|w_i|^p)^\frac{1}{p}$ is not admissible for any $p \geq 0, p \neq
2$, because it cannot be expressed as a function of the Hilbert space norm.
Indeed, if we choose any $a\in \R$ and let $w=(a \delta_{i1}: i \in \N_d)$,
the requirement that $\|w\|_p = h(\lb w,w \rb)$ would imply that
$h(a^2) = |a|,\forall a \in\R,$ and hence that $\|w\|_p = \|w\|$.

\section{Matrix Learning Problems}
\label{sec:matrix}

In this section, we investigate how representer theorems and results like Theorem
\ref{thm:main1} can be extended in the context of optimization problems
that involve matrices.

\subsection{Exploiting Matrix Structure}
\label{sec:matrix_intro}

As we have already seen, our discussion in Section \ref{sec:vector} 
applies to any Hilbert space. Thus, we may
consider the finite Hilbert space of $d \times n$ matrices $\Mb_{d,n}$ equipped with the
Frobenius inner product $\lb\cdot,\cdot \rb$.
As in Section \ref{sec:vector}, we could consider interpolation problems of the form
\beq
\label{eq:genMR}
\min \left\{ \Omega(W) : W \in M_{d,n}, \lb W, X_i \rb = y_i, i \in \N_m
\right\}
\eeq
where $X_i \in M_{d,n}$ are prescribed input matrices and
$y_i \in \calY$ scalar outputs, for $i \in \N_m$.
Then Theorem \ref{thm:main1} states that such a problem admits a solution of the
form
\begin{equation}
{\hat W} = \sum_{i\in\N_m} c_i X_i,
\label{eq:rep1}
\end{equation}
where $c_i$ are some real parameters, if and only if $\Omega$ can be written in
the form
\begin{align}
&&&& \Omega(W) = h(\lb W, W \rb) &&\forall~W\in \Mb_{d,n},
\label{eq:functional1}
\end{align}
where $h : \R_+ \rightarrow \R$ is nondecreasing.

However, optimization problems of the form \eqref{eq:genMR} do not
occur frequently in machine learning practice. The constraints
of \eqref{eq:genMR} do not utilize the structure inherent in matrices 
-- that is, it makes no difference whether the variable is regarded as a
matrix or as a vector --
and hence have limited applicability. In contrast, in many recent applications, 
some of which we shall briefly discuss below, it is natural to consider
problems like
\beq
\min \left\{ \Omega(W) : W \in \Mb_{d,n} \, , ~ w_t\trans x_{ti} = y_{ti} 
~~ \forall i \in \N_{m_t}, t \in \N_n \right\} \, .
\label{eq:mtl}
\eeq
Here, $w_t \in \R^d$ denote the columns of matrix $W$, for $t \in \N_n$,
and $x_{ti} \in \R^d, y_{ti} \in \calY$ are prescribed inputs and outputs,
for $i \in \N_{m_t}, t \in \N_n$.
In addition, the desired representation form for solutions of such matrix problems is
different from \eqref{eq:rep1}. In this case, one may encounter representer theorems of the form
\begin{align}
\label{eq:genRT}
&&&&&& {\hat w}_t = \sum_{s \in \N_n} \sum_{i \in \N_{m_s}} c^{(t)}_{si} x_{si} &&&& \forall t \in \N_n,
\end{align}
where $c^{(t)}_{si}$ are scalar coefficients for $s,t \in \N_n, i \in \N_{m_s}$.

To illustrate the above, consider the problem of multi-task learning and
problems closely related to it \citep[etc.]{bach_theo,mtl_feat,ml,spectral,candes,CCG_08,
Izenman,Maurer2,maurer,Srebro,ming}.
In learning multiple tasks jointly, each task may be represented by a
vector of regression parameters that corresponds to the column $w_t$
in our notation. There are $n$ tasks and $m_t$ data examples 
$\{ (x_{ti},y_{ti}) : i \in \N_{m_t} \}$ for the $t$-th task.
The learning algorithm used is
\beq
\min \left\{ \E \bigl( w_t\trans x_{t i}, y_{ti} : i \in \N_{m_t}, t \in \N_n\bigr)
+ \gamma \, \Omega(W): W \in \Mb_{d,n} \right\} \,,
\label{eq:mtl_reg}
\eeq
where $\E : \R^M \times \calY^M \to \R, M = \sum_{t\in\N_n}m_t$.
The error term expresses the objective that the regression vector for each task 
should fit well the data for this particular task.
The choice of the regularizer $\Omega$ is important in that it captures certain
relationships between the tasks.
One common choice is the {\em trace norm}, which is defined to be the
sum of the singular values of a matrix or, equivalently,
$$
\Omega(W) = \|W\|_1 := \trace (W\trans W)^\frac{1}{2} \,.
$$ Regularization with the trace norm learns the tasks as one joint
optimization problem, by favoring matrices with low rank. In other
words, the vectors $w_t$ are related in that they are {\em all} linear
combinations of a {\em small} set of basis vectors. It has been
demonstrated that this approach allows
for accurate estimation of related tasks even when there are only {\em
few} data points available for each task.

Thus, it is natural to consider optimization problems of the form
\eqref{eq:mtl}. In fact, these problems can be seen as instances of
problems of the form \eqref{eq:genMR},
because the quantity $w_t\trans x_{t i}$ can be
written as the inner product between $W$ and a matrix
having all its columns equal to zero except for the $t$-th
column being equal to $x_{t i}$.
It is also easy to see that \eqref{eq:genMR} is a richer class
since the corresponding constraints are less restrictive.

Despite this fact, by focusing on the class \eqref{eq:mtl} we concentrate
on problems of more practical interest and we can obtain representer theorems
for a richer class of regularizers, which includes the trace norm and
other useful functions. In contrast, regularization with the functional form
\eqref{eq:functional1} is not a satisfactory
approach since it ignores matrix structure. In particular,
regularization with the Frobenius norm (and a separable error
function) corresponds to learning each task {\em independently},
ignoring relationships among the tasks.

A representer theorem of the form \eqref{eq:genRT} for regularization
with the trace norm has been shown in \citep{ml}. Related results have
also appeared in \citep{bach_theo,srebro_icml}. We repeat here the
statement and the proof of this theorem, in order to better motivate our
proof technique of Section \ref{sec:matrix_proof}. 

\begin{theorem}
\label{thm:rep_trace_norm}
If $\Omega$ is the trace norm then problem \eqref{eq:mtl} (or problem \eqref{eq:mtl_reg}) 
admits a solution ${\hat W}$ 
of the form \eqref{eq:genRT}, for some
$c^{(t)}_{si} \in \R, ~i \in \N_{m_s}, s,t \in \N_n$. 
\end{theorem}
\begin{proof}
Let $\Wh$ be a solution of \eqref{eq:mtl} and let 
$\calL := \Span\{x_{si} : s \in \N_n, i \in \N_{m_s}\}$. We can
decompose the columns of $\Wh$ as $\wh_t = \wb_t + p_t, ~\forall t \in \N_n$, 
where $\wb_t \in \calL$ and $p_t \in \calL^\bot$. Hence $\Wh = \Wb + P$, 
where $\Wb$ is the matrix with columns $\wb_t$ and $P$ is the matrix with columns $p_t$. 
Moreover we have that $P\trans \Wb = 0$.  From Lemma \ref{lem:tracenorm} in the
appendix, we obtain that $\|\Wh\|_1 \geq \|\Wb\|_1$.
We also have that $\lb w_t, x_{ti} \rb = \lb \wb_t, x_{ti} \rb$, for every
$i \in \N_{m_t}, t \in \N_n$.
Thus, $\Wb$ preserves the interpolation constraints (or the value of the error term)
while not increasing the value of the regularizer. Hence, it is a solution of the
optimization problem and the assertion follows.
\end{proof}

A simple but important observation about this and related results is
that each task vector $w_t$ is a linear combination of the data for
{\em all} the tasks. This contrasts to the representation form
\eqref{eq:rep1} obtained by using Frobenius inner product constraints.
Interpreting \eqref{eq:rep1} in a multi-task context, by appropriately choosing 
the $X_i$ as described above, would imply that each $w_t$ is a
linear combination of only the data for task $t$.
 

Finally, in some applications the following 
variant, similar to the type \eqref{eq:mtl}, has appeared,
\begin{equation}
\min \left\{ \Omega(W) : W \in \Mb_{d,n} \,, ~ w_t\trans x_i = y_{ti} ~~ \forall i \in \N_m, t \in \N_n
\right\} \,.
\label{eq:matrix-int}
\end{equation}
Problems of this type corresponds to a special case in multi-task
learning applications in which the input points are the same for all
the tasks. For instance, this is the case with collaborative filtering
or applications in marketing where the same products/entities are
rated by all users/consumers (see, for example,
\citep{MSbook,EMP,lenk,Srebro} for various approaches to this problem).


\subsection{Characterization of Matrix Regularizers}
\label{sec:matrix_proof}

Our objective in this section will be to state and prove a general
representer theorem for problems of the form \eqref{eq:mtl} or 
\eqref{eq:matrix-int} using a
functional form analogous to \eqref{eq:functional_vector}. The key
insight used in the proof of \citep{ml} has been that the trace norm
is defined in terms of a matrix function that preserves the partial ordering of
matrices. That is, it satisfies Lemma \ref{lem:tracenorm}, 
which is a matrix analogue of the geometric property 
\eqref{eq:geom}. To prove our main result (Theorem \ref{thm:matrix2}), 
we shall build on this observation in a way similar to the approach
followed in Section \ref{sec:vector}.

Before proceeding to a study of matrix interpolation problems, it should be
remarked that our results will apply equally to matrix regularization
problems. That is, a variant of
Theorem \ref{thm:reg_int} can be shown for matrix regularization and
interpolation problems, following along the lines of the proof of that theorem.  The
hypothesis now becomes that for some $V,Y \in \Mb_{n,n}$, $V$
nonsingular, the minimizer of $\min\{\E(AV,Y) : A \in \Mb_{n,n}\}$ is
unique and nonsingular. As a result, matrix regularization with the square loss, the hinge loss
or the logistic loss does not differ from matrix interpolation
with respect to representer theorems.

Thus, we may focus on the interpolation problems \eqref{eq:mtl} and
\eqref{eq:matrix-int}. First of all, observe that, by definition, 
problems of the type \eqref{eq:mtl} include
those of type \eqref{eq:matrix-int}. Conversely, 
consider a set of constraints of the type
\eqref{eq:mtl} with one input per task ($m_t = 1, ~\forall t \in \N_n$)
and not all input vectors collinear. Then any matrix $W$
such that each $w_t$ lies on a fixed hyperplane perpendicular to $x_{t1}$ satisfies
these constraints. At least two of these hyperplanes do not coincide, whereas each constraint in
\eqref{eq:matrix-int} implies that all vectors $w_t$ lie on the same hyperplane.
Therefore, the class of problems \eqref{eq:mtl} is strictly larger than the class
\eqref{eq:matrix-int}.

However, it turns out that with regard to representer theorems of the form
\eqref{eq:genRT} there is no distinction between the two types of problems.
In other words, the representer theorem holds for the same
regularizers $\Omega$, independent of whether each task has its own
sample or not. More importantly, we can connect the existence of
representer theorems to a geometric property of the regularizer, in a
way analogous to property \eqref{eq:geom} in Section \ref{sec:vector}.
These facts are stated in the following lemma.

\begin{lemma}
The following statements are equivalent:
\begin{enumerate}
\item[{\rm (a):}] Problem \eqref{eq:matrix-int} admits a solution of the form \eqref{eq:genRT}, 
for every data set $\{(x_i,y_{ti}): i \in \N_m, t \in \N_n \} \subseteq
\Mb_{d,n} \times \Mb_{n,n}$ and every $m \in \N$, such that the interpolation constraints are satisfiable.

\item[{\rm (b):}] Problem \eqref{eq:mtl} admits a solution of the form \eqref{eq:genRT}, 
for every data set $\{(x_{ti},y_{ti}): i \in \N_{m_t}, t \in \N_n \} \subseteq
\R^d \times \R$ and every $m_t \in \N$, such that the interpolation constraints are satisfiable.

\item[{\rm (c):}] The function $\Omega$ satisfies the property
\begin{align} 
&& \Omega(W+P) \geq \Omega(W) && \forall~ W,P \in \Mb_{d,n} ~\text{such that}~ W\trans P = 0 \, .
\label{eq:matrix_geom0} 
\end{align}

\end{enumerate}
\label{lem:matrix1}
\end{lemma}

\begin{proof}
We will show that (a) $\implies$ (c), (c) $\implies$ (b) and (b) $\implies$ (a).

[(a) $\implies$ (c)] ~ Consider any $W \in \Mb_{d,n}$. 
Choose $m=n$ and the input data to be the columns of $W$.
In other words, consider the problem
$$
\min\{\Omega(Z): Z \in \Mb_{d,n}, Z\trans W = W \trans W \} \, .
$$ 
By hypothesis, there exists a solution ${\hat Z} = W C$ 
for some $C \in \Mb_{n,n}$. Since $({\hat Z} - W)\trans W = 0$,
all columns of ${\hat Z} - W$ have to belong to the null space
of $W$. But, at the same time, they have to lie in the range of $W$
and hence we obtain that ${\hat Z} = W$.
Therefore, we obtain property \eqref{eq:matrix_geom0} after the variable change 
$P = Z-W$. 

[(c) $\implies$ (b)] ~ Consider arbitrary $x_{ti} \in \R^d,y_{ti} \in \calY, 
i \in \N_{m_t}, t \in \N_n,$ and let $\Wh$ be a solution to problem \eqref{eq:mtl}. We can
decompose the columns of $\Wh$ as $\wh_t = \wb_t + p_t$ where
$\wb_t \in {\cal L} := {\rm span}\{x_{si}, i \in \N_{m_s}, s \in \N_n \}$,
and $p_t \in {\cal L}^\perp, ~ \forall t \in \N_n$.  By hypothesis $\Omega(\Wh) \geq 
\Omega(\Wb)$. Since $\Wh$ interpolates the data, so does $\Wb$ and therefore $\Wb$ is a
solution to \eqref{eq:mtl}.

[(b) $\implies$ (a)] ~ Trivial, since any problem of type \eqref{eq:matrix-int}
is also of type \eqref{eq:mtl}.
\end{proof}

The above lemma provides us with a criterion for characterizing all
regularizers satisfying representer theorems of the form
\eqref{eq:genRT}, in the context of problems \eqref{eq:mtl} or
\eqref{eq:matrix-int}.  
Our objective will be to obtain a functional form analogous to
\eqref{eq:functional_vector} that describes functions satisfying property
\eqref{eq:matrix_geom0}. This property does not have a simple geometric
interpretation, unlike \eqref{eq:geom} which describes functions with
spherical contours. The reason is that the matrix product in the
constraint is more difficult to tackle than an inner product. 


Similar to the Hilbert space setting \eqref{eq:functional_vector},
where we required $h$ to be a nondecreasing real function, the
functional description of the regularizer now involves the notion of a
{\em matrix nondecreasing} function.

\begin{definition}
\label{def:2}
We say that the function $h: \Sb^n_+ \rightarrow \R$ is nondecreasing
in the order of matrices if $h(A) \leq h(B)$ for all $A,B \in \Sb^n_+$
such that $A \preceq B$. \label{def:1}
\end{definition}

\begin{theorem}
Let $d,n \in \N$ with $d \geq 2n$. The differentiable function $\Omega : \Mb_{d,n} \to \R$
satisfies property \eqref{eq:matrix_geom0}  
if and only if there exists a matrix nondecreasing
function $h : \Sb_+^n \to \R$ such that 
\begin{align}
&&&& \Omega(W) = h(W\trans W), && \forall~ W\in\Mb_{d,n}.
\label{eq:functional_matrix}
\end{align}
\label{thm:matrix2}
\end{theorem}

\begin{proof}
We first assume that $\Omega$ satisfies property \eqref{eq:matrix_geom0}. 
From this property it follows that, for all $W,P \in \Mb_{d,n}$ with 
$W\trans P=0$,
\beq 
\lb \nabla \Omega(W), P\rb = 0. 
\label{eq:equiv2} 
\eeq
To see this, observe that if the matrix $W\trans P$ is zero then, 
for all $\varepsilon > 0$, we have that
$$
\frac{\Omega(W+\varepsilon P) - \Omega(W)}{\varepsilon} \geq 0.
$$
Taking the limit as $\varepsilon \to 0^+$ we obtain that $\lb \nabla
\Omega(W), P\rb \geq 0$. Similarly, choosing $\varepsilon < 0$ we obtain
that $\lb \nabla \Omega(W), P\rb \leq 0$ and equation \eqref{eq:equiv2}
follows.

Now, consider any matrix $W \in \Mb_{d,n}$. Let $r = \rank(W)$ and
let us write $W$ in a singular value decomposition as follows
\beq
W = \sum_{i\in \N_r} \sigma_i \, u_i v_i\trans \,,
\eeq
where $\sigma_1 \geq \sigma_2 \geq \dots \geq \sigma_r > 0$ are the
singular values and $u_i \in \R^d$, $v_i \in \R^n,$ $i \in \N_r$, sets of
singular vectors, so that $u_i\trans u_j = v_i \trans v_j =
\delta_{ij}$, $\forall i,j \in \N_r$. Also, let $u_{r+1}, \dots , u_{d} \in
\R^d$ be vectors that together with $u_1,
\dots, u_r$ form an orthonormal basis of $\R^d$.  Without loss of
generality, let us pick $u_1$ and consider any {\em unit} vector $z$
{\em orthogonal} to the vectors $u_2, \dots, u_r$. Let $k=d-r+1$ and $q
\in \R^{k}$ be the unit vector such that
\beq
z = R q,
\eeq
where $R =\left( u_1, u_{r+1},\dots,u_d \right)$. We can complete $q$
by adding $d-r$ columns to its right in order to form an orthogonal
matrix $Q \in \Ob^{k}$ and, since $d>n$, we may select these columns so that
$\det(Q) =1$.  Furthermore, we can write this matrix as $Q = e^D$ with
$D \in \Mb_{k,k}$ a skew-symmetric matrix (see \citep[Example
6.2.15]{horn2}).

We also define the path $Z: [0,1] \rightarrow \Mb_{d,n}$ as
\begin{align*}
&&&&&& Z(\lambda) = 
\sigma_1 
R e^{\lambda D} e_1 v_1\trans +
\sum_{i=2}^r \sigma_i \, u_i v_i\trans
&&&& \forall \lambda \in [0,1],
\end{align*}
where $e_1$ denotes the vector $\left(1, 0, \dots, 0 \right)\trans$.
In other words, we fix the singular values, the right singular vectors and the $r-1$ left 
singular vectors $u_2, \dots, u_r$ and only allow the first left singular vector to vary.
This path has the properties that $Z(0) = W$ and $Z(1) = 
\sigma_1 z v_1\trans + \sum_{i=2}^r \sigma_i \, u_i v_i\trans$.

By construction of the path, it holds that
$$
Z'(\lambda) = \sigma_1 R
e^{\lambda D} D e_1 v_1\trans
$$
and hence
$$ 
Z(\lambda)\trans Z'(\lambda) =
\left( \sigma_1 R
e^{\lambda D} e_1 v_1\trans \right)\trans
\sigma_1 R
e^{\lambda D} D e_1 v_1\trans = \sigma_1^2 \, v_1 e_1\trans D e_1 v_1\trans = 0 \,,
$$
for every $\lambda \in [0,1]$, because $D_{11}=0$. 
Hence, using equation \eqref{eq:equiv2},
we have that
$$
\lb\nabla \Omega(Z(\lambda)), Z'(\lambda)\rb = 0
$$
and, since $\dfrac{d \Omega(Z(\lambda))}{d \lambda} = \left\langle\nabla \Omega(Z(\lambda)),
Z'(\lambda)\right\rangle$, we conclude that $\Omega(Z(\lambda))$ equals a constant
independent of $\lambda$. In particular, $\Omega(Z(0)) = \Omega(Z(1))$, that
is, 
$$
\Omega(W) = 
\Omega \left(\sigma_1 z v_1\trans + \sum_{i=2}^r \sigma_i \, u_i v_i\trans \right) \,.
$$ In other words, if we fix the singular values of $W$, the right
singular vectors and all the left singular vectors but one, $\Omega$
does not depend on the remaining left singular vector (because the
choice of $z$ is independent of $u_1$).

In fact, this readily implies that $\Omega$ does not depend on the
left singular vectors at all. Indeed, fix an arbitrary $Y \in
\Mb_{d,n}$ such that $Y\trans Y=I$. Consider the matrix 
$Y(W\trans W)^{\frac{1}{2}}$, which can be written using the same
singular values and right singular vectors as $W$. That is, $$
Y(W\trans W)^{\frac{1}{2}} = \sum_{i\in \N_r} \sigma_i \, \tau_i
v_i\trans \,, $$ where $\tau_i = Y v_i,~\forall i \in \N_r$. Now, we select
unit vectors $z_1, \dots, z_r$ as follows:
\begin{align}
z_1 & = u_1 \\
z_2 & \perp z_1, u_3, \dots, u_r, \tau_1 \\
& \vdots \quad \vdots \\
z_r & \perp z_1, \dots, z_{r-1}, \tau_1, \dots, \tau_{r-1} \,.
\end{align}
This construction is possible since $d\geq 2n$. Replacing
successively $u_i$ with $z_i$ and then $z_i$ with $\tau_i$, 
$\forall i \in \N_r$, and applying the invariance property, we obtain that
\begin{eqnarray}
\Omega(W) & = & \Omega \left(\sum_{i\in\N_r} \sigma_i \, u_i v_i\trans \right) \\
~& = &\Omega \left( \sigma_1 z_1 v_1\trans + \sigma_2 z_2 v_2\trans + 
\sum_{i=3}^r \sigma_i \, u_i v_i\trans \right) \\
~& ~ & \qquad \vdots \qquad \vdots \\
~& = &\Omega \left( \sum_{i\in\N_r} \sigma_i \, z_i v_i\trans \right) \\
~& = & \Omega \left( \sigma_1 \, \tau_1 v_1\trans + 
\sum_{i=2}^r \sigma_i \, z_i v_i\trans \right) \\
~& ~ & \qquad \vdots \qquad \vdots \\
~& = & \Omega \left( \sum_{i\in\N_r} \sigma_i \, \tau_i v_i\trans \right) = \Omega \left( Y(W\trans W)^{\frac{1}{2}} \right) \,.
\end{eqnarray}
Therefore, defining the function $h: \Sb_+^n \rightarrow \R$ as $h(A) =
\Omega(Y A^{\frac{1}{2}})$, we deduce that $\Omega(W) = h(W\trans W)$.

Finally, we show that $h$ is matrix nondecreasing, that is,
$h(A) \leq h(B)$ if $0 \preceq A \preceq B$. For any such $A,B$ and since $d \geq 2n$, we may define 
$W = [A^{\frac{1}{2}},0,0]\trans$,
$P = [0,(B-A)^{\frac{1}{2}},0]\trans$ $\in \Mb_{d,n}$.
Then $W\trans P = 0$, $A = W\trans W$, $B = (W+P)\trans (W+P)$ and thus, by
hypothesis,
$$
h(B) = \Omega(W+P) \geq \Omega(W) = h(A).
$$
This completes the proof in one direction of the theorem.

To show the converse, assume that $\Omega(W) = h(W\trans W)$, where the function $h$ 
is matrix nondecreasing. Then for any $W,P \in \Mb_{d,n}$
with $W\trans P =0$, we have that
$(W+P)\trans (W+P) = W\trans W + P\trans P \succeq W\trans W$ and, so,
$\Omega(W+P) \geq \Omega(W)$, as required.
\end{proof}

We conclude this section by providing a necessary and 
sufficient condition on the matrix nondecreasing property of the function $h$.


\begin{proposition}
Let $h: \Sb_+^n \rightarrow \R$ be differentiable function. The following properties are equivalent: 
\begin{enumerate}
\item[(a)] $h$ is matrix nondecreasing

\item[(b)] the matrix $\nabla h(A) := \left( \frac{\partial h}{\partial a_{ij}}: i,j
\in \N_n \right)$ is positive semidefinite, for every $A\in \Sb^n_+$. 
\end{enumerate}
\label{lem:1}
\end{proposition}
\begin{proof}
If (a) holds, we choose $x \in \R^n$, $t \in \R$ and note that $$
\frac{h(A+ t x x\trans) - h(A)}{t} \geq 0.
$$
Letting $t$ go to zero gives that $x\trans \nabla h(A) x \geq 0$.

Conversely, if (b) is true we have, for every $x \in \R^n$, that $x\trans \nabla h(A) x
= \lb\nabla h(A), xx\trans\rb \geq 0$ and, so, $\lb\nabla h(A), C\rb \geq 0$ for all $C \in \Sb_+^n$.
For any $A, B \in \Sb^n_+$ such that $A \preceq B$, consider the univariate function 
$g: [0,1] \rightarrow \R$, $g(t) = h(A+t(B-A))$. By the chain rule it is easy to verify 
that $g$ in nondecreasing. Therefore we conclude that $h(A)=g(0) \leq g(1) = h(B)$.
\end{proof}


\subsection{Examples}
\label{sec:examples}
We have briefly mentioned already that functional description
\eqref{eq:functional_matrix} subsumes the special case of {\em monotone spectral
functions}. By spectral functions we simply mean those real-valued
functions of matrices that depend only on the singular values of their
argument. Monotonicity in this case simply means that one-by-one
orderings of the singular values are preserved. In addition, the
monotonicity of $h$ in \eqref{eq:functional_matrix} is a direct
consequence of Weyl's monotonicity theorem \citep[Cor. 4.3.3]{horn},
which states that if $A\preceq B$ then the spectra of $A$ and $B$ are
ordered.

Interesting examples of such functions are the {\em Schatten} $L_p$
{\em norms} and {\em prenorms}, $$
\Omega(W) = \|W\|_p := \|\sigma(W)\|_p \,,
$$ 
where $p \in [0,+\infty)$ and $\sigma(W)$ denotes the
$n$-dimensional vector of the singular values of $W$. For instance, we
have already mentioned in Section \ref{sec:matrix_intro} that the
representer theorem holds when the regularizer is the trace norm (the
$L_1$ norm of the spectrum).  But it also holds for the {\em rank} of
a matrix, which is the $L_0$ prenorm of the spectrum.  Regularization
with the rank is an NP-hard optimization problem but the representer
theorem implies that it can be solved in time dependent on the total
sample size.

If we exclude spectral functions, the functions that remain are
invariant under {\em left} multiplication with an orthogonal
matrix. Examples of such functions are Schatten norms and prenorms
composed with {\em right} matrix scaling, 
\beq
\label{eq:addmassi}
\Omega(W) = \|W M \|_p \,,
\eeq
where $M \in \Sb^n$. In this case, the corresponding $h$ is the
function $S \mapsto \|\sqrt{\sigma(M S M)}\|_p$. To see that this
function is matrix nondecreasing, observe that if $A,B \in \Sb^n_+$
and $A\preceq B$ then $0 \preceq M A M \preceq M B M$ and hence
$\sigma(MAM) \preceq \sigma(MBM)$ by Weyl's monotonicity
theorem. Therefore, $\|\sqrt{\sigma(MAM)}\|_p \leq
\| \sqrt{\sigma(MBM)} \|_p$.

Also, the matrix $M$ above can be used to select a subset of the
columns of $W$. In addition, more complicated structures can be
obtained by summation of matrix nondecreasing functions and by taking
minima or maxima over sets. For example, we can obtain a regularizer
such as $$
\Omega(W) = \min_{\{I_1,\dots,I_K\} \in {\cal P}} \sum_{k\in\N_K} \|W(I_k)\|_1  \, ,
$$ 
where ${\cal P}$ is the set of partitions of $\N_n$ in $K$ subsets
and $W(I_k)$ denotes the submatrix of $W$ formed by just the columns
indexed by $I_k$. This regularizer is an extension of the trace norm
and can be used for learning multiple tasks via dimensionality
reduction on more than one subspaces \citep{ecml}. 

Yet another example of valid regularizer is that considered in \cite[Sec.~3.1]{EMP}, 
which encourages the tasks to be close to each others, namely
$$
\Omega(W) = \sum_{t=1}^n \left\|w_t - \frac{1}{n} \sum_{s=1}^n w_s\right\|^2.
$$ This regularizer immediately verifies property
\eqref{eq:matrix_geom0}, and so by Theorem \ref{thm:matrix2} 
it is a matrix nondecreasing function of $W\trans W$. One can also verify that 
this regularizer is the square of the form \eqref{eq:addmassi} with $p=2$.


Finally, it is worth noting that the representer theorem does {\em
not} apply to a family of ``mixed'' matrix norms that have been used
in both statistics and machine learning, in formulations such as the
``group Lasso''
\citep{antoniadis,mtl_feat,bakin,Canu,cosso,Obozinski,group_lasso}.  These
norms are of the form $$
\Omega(W) = \|W\|_{p,q} := 
\left( \sum_{i\in\N_d} {\|w^i\|_p}^q \right) ^{\frac{1}{q}} \,,
$$ 
where $w^i$ denotes the $i$-th row of $W$ and $(p,q) \neq (2,2)$.
Typically in the literature, $q$ is chosen equal to one in order to
favor sparsity of the coefficient vectors {\em at the same covariates}.

\section{Conclusion}
We have characterized the classes of vector and matrix regularizers
which lead to certain forms of the solution of the associated
regularization problems. In the vector case, we have proved the
necessity of a well-known sufficient condition for the ``standard
representer theorem'', which is encountered in many learning and
statistical estimation problems. In the matrix case, we have described
a novel class of regularizers which lead to a modified representer
theorem. This class, which relies upon the notion of matrix
nondecreasing function, includes and extends significantly the vector
class. To motivate the need for our study, we have discussed some
examples of regularizers, which have been recently used in
the context of multi-task learning and collaborative filtering.

In the future, it would be valuable to study more in detail special
cases of the matrix regularizers which we have encountered, such as
those based on orthogonally invariant functions. It would also be
interesting to investigate how the presence of additional constraints
affects the representer theorem. In particular, we have in mind the
possibility that the matrix may be constrained to be in a convex cone,
such as the set of positive semidefinite matrices. Finally, we leave
to future studies the extension of the ideas presented here to the
case in which matrices are replaced by operators between two Hilbert
spaces.

\section*{Acknowledgments}
The work of the first and third authors was supported by EPSRC Grant
EP/D052807/1 and by the IST Programme of the European Community, under
the PASCAL Network of Excellence IST-2002-506778. The second author is
supported by NSF grant DMS 0712827.


\section*{Appendix}

Here we collect some auxiliary results which are used in the above
analysis.

The first result states a basic property of connectedness through
rotations.
\begin{lemma}
Let $w,v \in \R^d$ and $d \geq 2$. Then there exists $U \in \Ob^d$
with determinant $1$ such that $v = U w$ if and only if $\|w\| =
\|v\|$.
\label{lem:3}
\end{lemma}

\begin{proof}
If $v = U w$ we have that $v\trans v = w\trans w$.
Conversely, if  $\|w\| = \|v\|$, we may
choose orthonormal vectors $\{x_\ell: \ell \in \N_{d-1}\} \perp w$ 
and $\{z_\ell : \ell \in \N_{d-1}\} 
\perp v$ and form the matrices
$R=\begin{pmatrix}w,x_1,\dots,x_{d-1}\end{pmatrix}$ and
$S=\begin{pmatrix}v,z_1,\dots,z_{d-1}\end{pmatrix}$. We have that $R\trans R = S\trans
S$. We wish to solve the equation $U R = S$. For this purpose we
choose $U=S R^{-1}$ and note that $U \in \Ob^d$ because $U\trans U =
(R^{-1})\trans S^T S R^{-1} = (R^{-1})\trans R\trans R R^{-1} = I$.
Since $d \geq 2$, in the case that $\det(U) = -1$ we can simply 
change the sign of one of the $x_\ell$ or $z_\ell$ to get $\det(U) = 1$
as required. 
\end{proof}

The second result concerns the monotonicity of the trace norm.
\begin{lemma}
Let $W,P \in \Mb_{d,n}$ such that $W\trans P=0$. Then $\|W+P\|_1 \geq \|W\|_1$. 
\label{lem:tracenorm}
\end{lemma}

\begin{proof}
It is known that the square root function, $t \mapsto t^\frac{1}{2}$, is {\em matrix monotone}
-- see, for example, \cite[Sec. V.1]{bhatia}.
This means that for any matrices $A,B \in \Sb^n_+$, $A\succeq B$ implies 
$A^\frac{1}{2} \succeq B^\frac{1}{2}$.
Hence, for any matrices $A,B \in \Sb^n_+$, $A\succeq B$ implies
$\trace A^\frac{1}{2} \geq \trace B^\frac{1}{2}$.
We apply this fact to the matrices $W\trans W + P\trans P$ and $P\trans P$ to obtain that
\begin{eqnarray*}
\|W+P\|_1 = \trace((W+P)\trans (W+P))^{\frac{1}{2}} =
\trace(W\trans W + P\trans P)^{\frac{1}{2}} \geq \trace(W\trans W)^{\frac{1}{2}} = \|W\|_1 \,.
\end{eqnarray*}
\end{proof}


\bibliography{representer}

\end{document}